 \documentclass[lettersize,journal]{IEEEtran}
\usepackage{amsmath,amsfonts}
\usepackage{algpseudocode}
\usepackage{algorithm}
\usepackage[hidelinks]{hyperref}
\usepackage{array}
\usepackage[caption=false,font=normalsize,labelfont=sf,textfont=sf]{subfig}
\usepackage{textcomp}
\usepackage{cite}
\usepackage{amsthm}
\newtheorem{theorem}{Theorem}
\newtheorem{lemma}[theorem]{Lemma}
\usepackage{stfloats}
\usepackage{url}
\usepackage{verbatim}
\usepackage{graphicx}
\usepackage{cite}
\usepackage{amssymb}
\begin{document}

\title{Space Filling Curves for Coverage Path Planning with Online Obstacle Avoidance}

\author{Ashay Wakode$^{1}$, Arpita Sinha$^{2}$
\thanks{$^{1}$ Ashay Wakode is with Autonomous Robotics Research Center, Technology Innovation Institute, Abu Dhabi, United Arab Emirates
{\tt\small  {ashaywakode}@gmail.com}}

\thanks{$^{2}$ Arpita Sinha is with System and Controls Department; Indian Institute Of Technology Bombay, Mumbai, India  \\
{\tt\small  {arpita.sinha}@iitb.ac.in  }
}

}


\maketitle
\begin{abstract}
The paper presents a strategy for robotic exploration problem using Space-Filling curves (SFC). The strategy plans a path that avoids unknown obstacles while ensuring complete coverage of the free space in region of interest. The region of interest is first tessellated, and the tiles/cells are connected using a SFC pattern. A robot follows the SFC to explore the entire area. However, obstacles can block the systematic movement of the robot. We overcome this problem by determining an alternate path online that avoids the blocked cells while ensuring all the accessible cells are visited at least once. The proposed strategy chooses next waypoint based on the graph connectivity of the cells and the obstacle encountered so far. It is online, exhaustive and works in situations demanding non-uniform coverage. The completeness of the strategy is proved and its desirable properties are discussed with examples.  

\end{abstract}
\section{Introduction}
\IEEEPARstart{I}{n} 1878, George Cantor demonstrated that an interval $I = [0,1]$ can be mapped bijectively onto $[0,1] \times[0,1]$. Later, G. Peano discovered one such mapping that is also continuous and surjective; the image of such mapping when parameterized in the interval $I$ to higher dimensions ($\mathbb{R}^{n}$) is known as Space-Filling Curve (SFC). More SFCs were later discovered by E. Moore, H. Lebesgue, W. Sierpinski, and G. Polya \cite{sagan,bader}. Space-filling curves (SFCs) possess intriguing properties. They are self-similar, meaning each curve consists of sub-curves similar to the entire curve. SFCs are also surjective maps, ensuring they cover every point in $\mathbb{R}^{n}$. Furthermore, they preserve locality, so points close in $I$ remain close in $\mathbb{R}^{n}$.
\par
Due to the above properties, SFCs have been used in many applications - data collection from sensor network \cite{sensor1, sensor2}; ordering meshes of complex geometries \cite{bader} and many more. An approximate solution to Travelling Salesman Problem (TSP) can be found using Hilbert's Space Filling curve \cite{tsp}. Space Filling Tree analogous to SFCs having tree-like structure have been proposed for sampling-based path planning \cite{SFCtree}, as opposed to traditional methods like Rapid-exploring Random Trees (RRTs) \cite{RRT}.
\par
In robotic exploration problem, a single or group of robotic agents are deployed to search, survey or gather information about a specific region while avoiding obstacles. Robotic exploration is one of the sub-problems of the larger Coverage Planning Problem (CPP), wherein the agent is bestowed with the task of visiting all points in a 2D area or volume \cite{galceran, choset, uavsurvey, VCPP}. Numerous approaches for CPP already exists - Graph based, Grid based, Neural-Network based, Cellular decomposition based \cite{galceran}. Each of these approaches can be used for robotic exploration problem. 
\par
SFCs have been utilized for robotic exploration, with their suitability stemming from the advantageous properties they possess. Exploration using SFCs is time complete and robust to failure \cite{spires}. Hilbert's curve have been shown to be more efficient in time of travel / cost of task than lawnmower's path \cite{sadat2}. The exploration strategies developed for 2D using SFCs can be extended to 3D since similar grammar exists for their construction in both dimensions \cite{bader}. The generalized version of SFCs aka Generalized SFC (GSFC) can span irregular quadrilateral and triangles \cite{bader}. SFCs can be easily used in non-uniform coverage scenarios requiring some parts to be searched more rigorously than others.
 \par
There is sizable literature dealing with SFC based robotic exploration.  \cite{spires} suggests the use of a swarm of mobile robots for exploration, each mobile robot following a SFC. This approach is efficient in terms of energy, robust to failures and assures coverage in finite time, but considers an obstacle-free environment. \cite{sadat1} proposes a novel method for non-uniform coverage planning for UAVs. \cite{sadat2} takes the work further and uses the Hilbert curve for non-uniform coverage which is optimal as opposed to existing methods not using SFC. \cite{hawk} introduced a UAV system to conduct aerial localization that uses Moore's SFC. However, \cite{sadat1, sadat2, hawk} does not consider obstacles. 
\par
\cite{tiwari} introduced path planning approach for SFCs with obstacles and proved the optimality for specific obstacle configurations, due to the existence of a Hamiltonian path for such obstacle configurations. \cite{ban} introduced an algorithm to construct SFCs for sensor networks with holes. The algorithm can be used for motion planning with obstacles while using SFC. However, the solutions proposed in \cite{tiwari, ban} require the knowledge of obstacles before starting the exploration.
\par 
\cite{nair} formulated an online obstacle evasion strategy for Hilbert curve with only one waypoint blocked by an obstacle. \cite{joshi} builds upon \cite{tiwari} and suggests a strategy capable of evading two neighboring waypoints on a Hilbert curve. \cite{wakode} talks about online obstacle avoidance for aerial vehicle covering a region using the Sierpinski curve, but the obstacles need to block disjoint cells.
\par
Early work on the topic did not consider obstacles in the environment. Later works considered obstacles but required knowledge about the obstacles or were restricted to a particular class of obstacles and SFC used. Hence, the necessary next step is to develop an online strategy that can navigate around arbitrary obstacles while employing any SFC and this is addressed by the proposed strategy, The contributions of the paper are:
\begin{itemize}
    \item The proposed strategy is online and can avoid obstacles on he go using the grid connectivity of the SFC cells. 
    \item The strategy is applicable to any SFC and any number of obstacles placed randomly in the environment.
    \item The strategy can be used for non-uniform coverage.
\end{itemize}
\par
The rest of this paper is organized into five sections: Section II formally introduces SFCs. Section III formulates the problem and the proposed strategy is presented. Section IV discusses the properties of the strategy and the implemented examples. Section V concludes the paper and talks about the limitations and possible future work. 

\section{Preliminaries}
Space Filling curves (SFCs) are defined as follows, \\
\textbf{Definition} \cite{bader}: Given a mapping $f : I \rightarrow Q \subset \mathbb{R}^{n}$, with $f_{*}\textit(I)$ as image, $f_{*}\textit(I)$ is called a SFC, if $f_{*}\textit(I)$ has Jordan content (area, volume, ..) larger than 0. 
\par 
In this paper, approximations of SFCs are used. Approximate SFCs are constructed by dividing $\mathbb{R}^{n}$ and connecting the centers of the cells by a continuous piecewise straight line. The way the divisions are connected depends on the grammar of the SFC being used.
\par
\begin{figure}[t]
    \centering
    \includegraphics[width=3in]{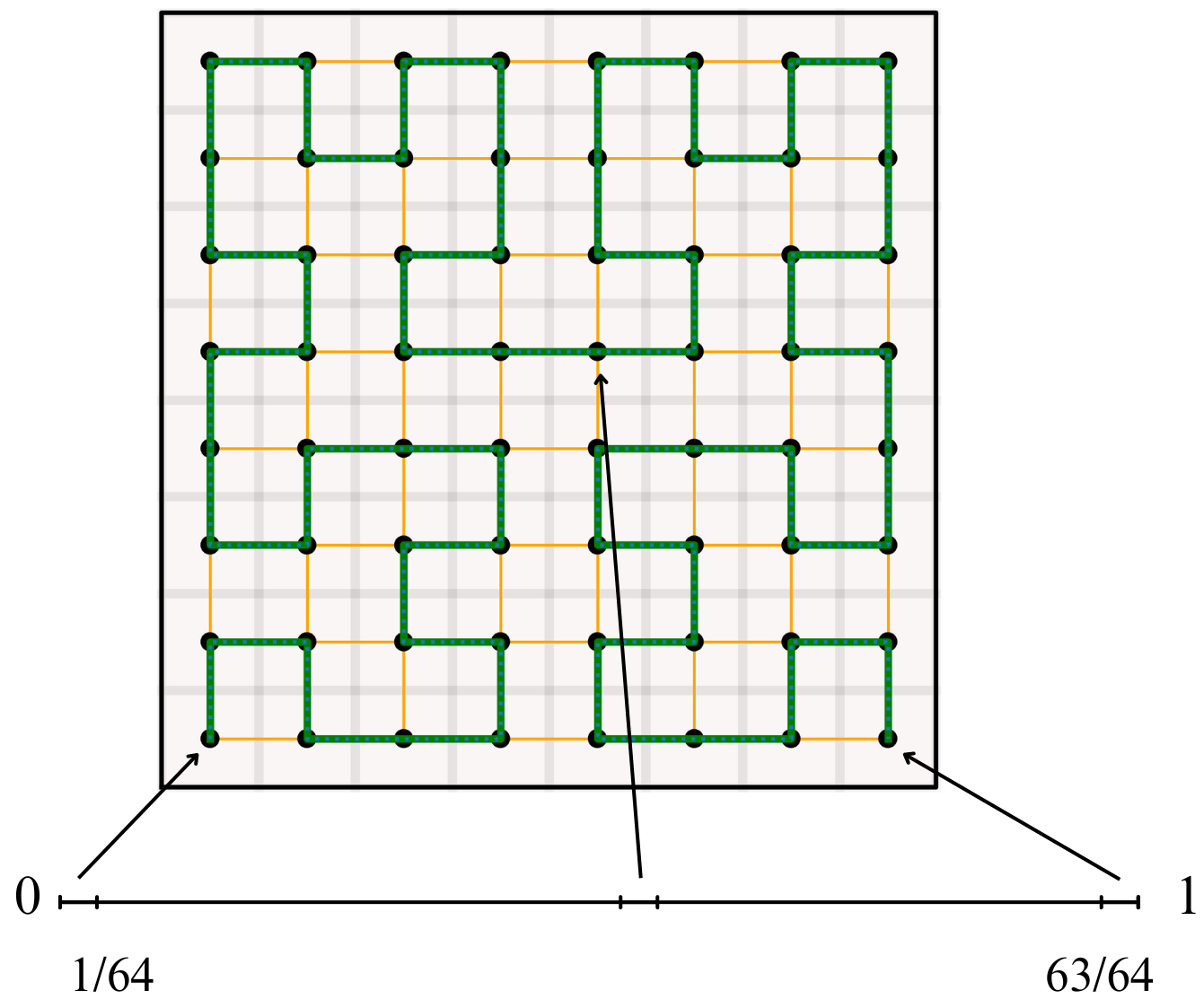}
    \caption{Hilbert curve as mapped from $I$ to a tessellated square}
    \label{fig:mapping}
\end{figure}

\par
Figure \ref{fig:mapping} shows a square divided N times and the centers are connected to generate Hilbert curve. In this case, hilbert curve is mapped onto a square in $\mathbb R^{2}$. The degree to which the square is divided is quantified by the term ``iteration" (also referred to as order). As iteration goes to infinity, the SFC visits every point in the square. This explains the \textbf{surjective mapping} property of SFCs. A rigorous mathematical explanation can be found in \cite{sagan,bader}. The approximate SFCs are simply referred to as SFC in literature. Moreover, approximate SFCs are sufficient for the exploration given that the robotic agent has a search radius enough to cover the entire cell (area or volume) while at its center. Furthermore, we can see in Fig \ref{fig:iterations} that iteration $N$ is constructed by rotation and translation of iteration $N-1$, which demonstrates the \textbf{self-similar} property. The grammatical way of constructing SFCs is based on this fact.
\par SFCs are H\"{o}lder continuous mappings. Given, 
\begin{equation*}
    x, y \in I
\end{equation*}
and $f$ is SFC with $f : I \rightarrow Q \subset \mathbb{R}^{n}$, if,
\begin{equation*}
     \lVert f(x)- f(y) \rVert _{2} \leq C |x-y|^{r} 
\end{equation*}
\begin{figure}[t]
    \centering
    \includegraphics[width=2.8in]{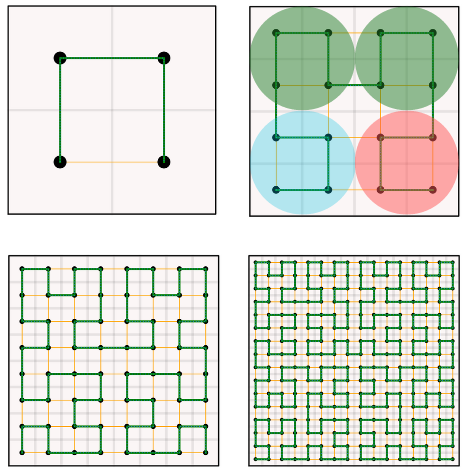}
    \caption{Hilbert curve iterations 1 to 4; Colored circles represent specific translation + rotation rules for creating $n^{th}$ iteration from $n-1^{th}$ iteration}
    \label{fig:iterations}
\end{figure}
The function $f$ is said to be H\"{o}lder continuous with exponent $r$, if a constant $C$ exists for all $x,y$. The RHS can be interpreted as the distance between $x$ and $y$, while LHS is the distance between points in SFC. Here, the distance between points on SFC is bounded by the distance between the points they are mapped from, explaining the \textbf{Locality Preservation}.
Lastly, any problem in Q $\subset \mathbb{R}^{n}$ mapped by SFC can be transformed into a problem in $I$. And often solving the problem in $I$ is straightforward than solving it in the original space $Q$. This technique is referred to as \textbf{General Space-Filling Heuristics} (\textbf{GSFH}).
\par
\section{Problem Formulation \& Proposed Solution}
We consider a robotic agent with sensing radius $s$ that needs to explore a 2D region. This region is partitioned into cells so that each cell remains within the agent's sensor footprint when positioned at the cell's center. The robot traverses the cells sequentially following the SFC pattern. It identifies obstacles while moving into the cell containing them and stores the obstacle information for future reference.
\par
We elaborate our strategy using the Hilbert curve, but the strategy can be used for any SFC. Consider a square with area $A$, having static obstacles at unknown locations. The number of obstacles are not known and they can have arbitrary shapes and sizes. The iteration ($k \in N$) of Hilbert curve is determined such that the agent can scan an entire cell while at its center. Mathematically the iteration $k$ can be found out by comparing sensing radius $s$ and diagonal of cell for Hilbert curve, 
\begin{equation*}
 \label{eq1}
    k \geq \lceil log_{2}(A / s \sqrt{2} - 1) \rceil
\end{equation*}
\par
The center of each cell is identified as waypoint and are numbered from $0$ to $N-1$ starting from one of the corner cells. Without loss of genrality, we assume it to be the left-bottom cell. The goal of the robotic agent is to start at $0$ and visit all the unblocked reachable waypoints.  
\par 
Here, GSFH simplifies a 2D scanning problem into a 1D routing task by employing SFC. When addressing the 1D scenario which involves covering all waypoints along a line, we adopt a greedy approach by sequentially moving from left to right, ensuring all points to the left are visited before proceeding right. This method avoids the need for backtracking to reach any previously skipped waypoint on the left, thereby ensuring an efficient path is followed. In a 1D routing task with obstacle (or obstacles), the agent will be limited to traversing only the segment of the line from its starting point. 
\par
In our scenario, traversing the 1D route left to right or visiting waypoints $0$ to $N-1$ sequentially the agent effectively navigates the 2D region following the SFC pattern. The existence of an obstacle does not pose a significant problem, as the agent can deviate from the SFC pattern and identify an alternate route through other connected waypoints.
\par 
\begin{figure}[t]
    \centering
    \includegraphics[width=3in]{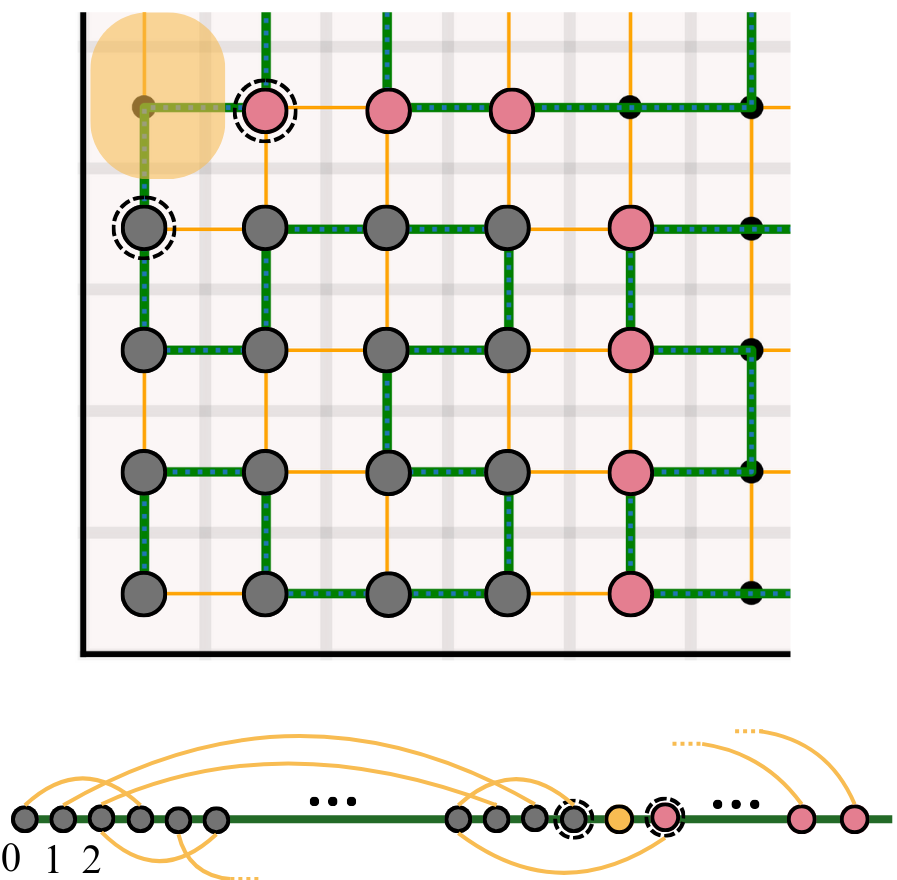}
    \caption{Grey dots: Visited vertices of the graph, Yellow lines: Edges connecting the vertices, Red dots: Vertices adjacent to already visited vertices}
    \label{fig:explain}
\end{figure}
The proposed strategy considers the waypoints (centers of the cells) as vertices of graph, with adjacent waypoints connected. It determines the next waypoint/vertex to visit based on vertices already visited and obstacles identified (see figure \ref{fig:explain}). The agent encounters an obstacle while at dotted grey vertex and chooses to visit dotted red vertex among other adjacent vertices. The strategy selects the lowest-numbered vertex that has not yet been visited and is adjacent to the set of vertices already visited. A shortest path is determined to the selected vertex. The strategy is used repeatedly until all the vertices connected to initial one are visited atleast once. The agent rejects to choose the vertices where obstacle has been detected in the past.
\par 
We use the following graph theoretic nomenclature for our problem:
\begin{itemize}
    \item $G = $ Graph with waypoints of SFC as vertices; Adjacent vertices are connected by the edge; No waypoint is assumed to be blocked by an obstacle in $G$. $G$ is a dual graph of SFC decomposition.
    \item $O = $ Set of vertices where obstacles have been detected  
    \item $V = $ Set of visited vertices 
    \item $A(V, G) = $ Set of vertices in $G$ adjacent to vertices in $V$ but not present in $V$
\end{itemize}

\par
The strategy is presented as a pseudocode in Algorithm \ref{alg:main}. While at waypoint $c$, the agent knows $G$ given the SFC, $V$ from the visited nodes and $O$ from the detected obstacles till that time. If all the waypoints adjacent to $V$ in $G$ are blocked, no waypoint remains to be visited and the search can be terminated (line $4$). On the other hand, the strategy suggests minimum numbered waypoint $p$, and the shortest path $R$ to $p$ is found using Dijkstra's algorithm (any shortest path finding algorithm can be used). $R$ refers to an array of waypoints starting with $c$ and terminating at $p$. The agent checks if $p$ is blocked while at the pen-ultimate waypoint (line $7$). If $p$ is found to be blocked, it is added to $O$ and the strategy starts again at step $4$ (line $8$ to $10$). Otherwise, $p$ is reached and added to $V$ (line $12$ and $13$). The algorithm terminates when no waypoint adjacent to the visited waypoints remain.
\begin{algorithm}
\caption{Online Strategy for Obstacle evasion}\label{alg:main}
\begin{algorithmic}[1]
\State Input : $G$, $V$, $O$
\State Output : Next waypoint
\State Initialization : Current waypoint ($c$)
\If{$A(V, G) - O \neq \emptyset$}
    \State $p = min(A(V, G) - O)$ \Comment{min numbered element}
    \State $R =$ shortest route to $p$ starting from $c$
    \While{Agent at $R[-2]$} \Comment{second last element of $R$}
        \If{$p$ is blocked $\lor$  $p \in O$} 
        \State $O \gets O \cup \{p\}$  
        \State Go back to step $4$
        \Else 
        \State Next waypoint $= p$ 
        \State  $V \gets V \cup \{p\}$
        \EndIf
        \EndWhile
\Else 
\State All the reachable waypoints visited
        \EndIf
\end{algorithmic}
\end{algorithm}

\begin{lemma}
An agent starting at waypoint $H$ and following the proposed strategy will visit all the waypoints connected to $H$.
\end{lemma}
\begin{proof}
We prove the lemma by contradiction. Assume a waypoint $J$ is not blocked by an obstacle and is connected to $H$, but the agent terminated the exploration without visiting $J$. This can happen only if,
\begin{equation*}
    A(V, G) - O = \emptyset
\end{equation*}
But, 
\begin{equation*}
    A(V, G) - O = \{J\}
\end{equation*}
Since, $J$ is connected to $H$ and not visited. Therefore, a contradiction.
\end{proof}
\begin{figure}[t]
\centering
\includegraphics[width=3in]{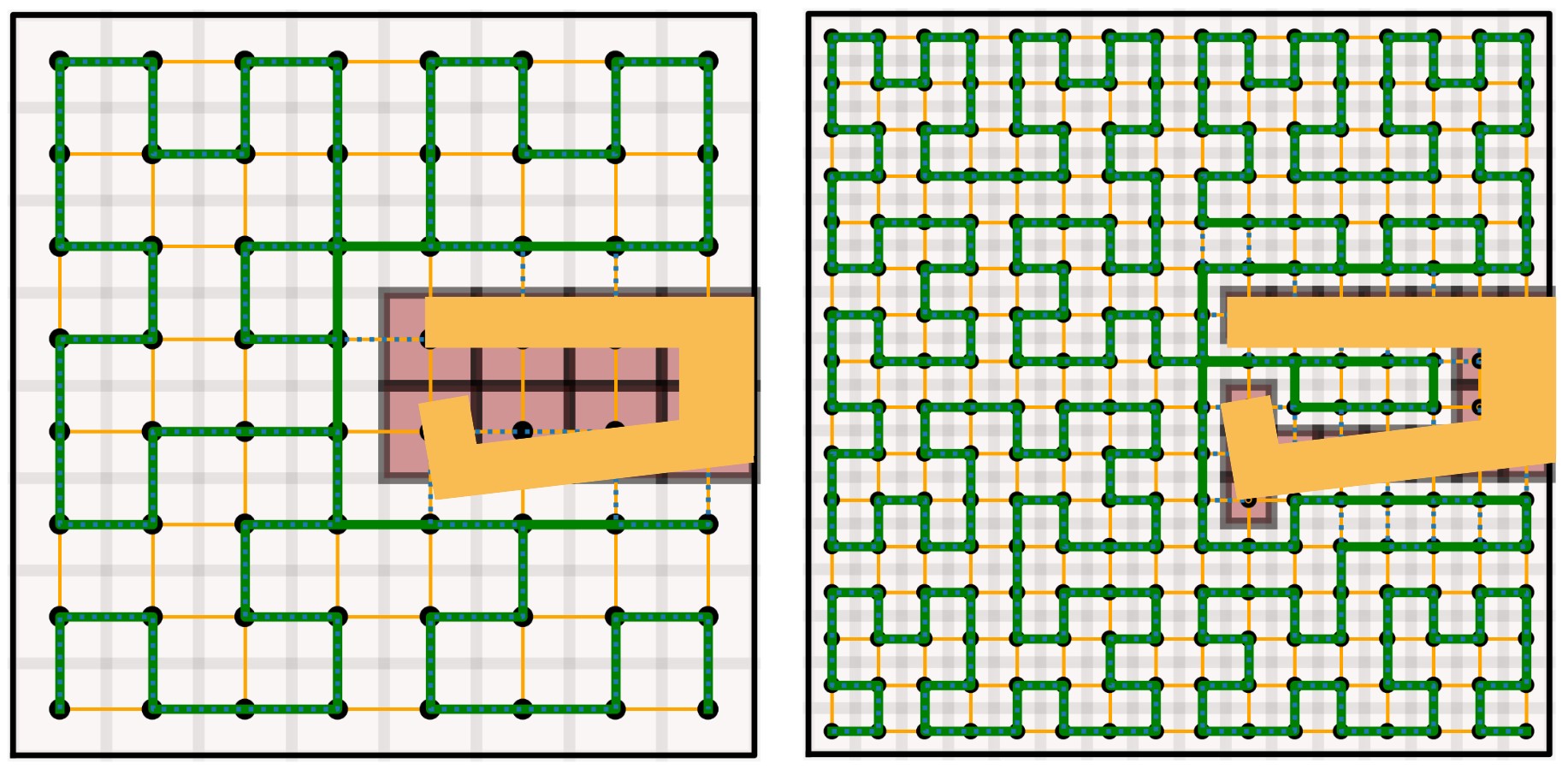}
\caption{Region blocked by tight space reachable through the use of higher iteration (4)}
\label{inv}
\end{figure}
The agent will explore the area spanned by waypoints given the sensing radius is enough to cover the entire cell (Eq. \ref{eq1}). In scenarios involving confined areas, the iteration of SFC might result in blocking access to certain areas that are otherwise reachable. These instances can be detected by comparing the sum of visited and blocked waypoints to the highest numbered (say $c^\prime$) waypoint achieved at the end of the search. 
\par 
If,
\begin{equation*}
    n(V) + n(O) < c^\prime
\end{equation*}
This could mean some waypoints have been blocked due to confined spaces which can be tackled by increasing the iteration of SFC (see figure \ref{inv}), or some waypoints are obstructed on all sides and any change in search will not matter. 

\section{Discussion \& Results}
In this paper, the obstacles bigger or comparable to the sensing radius of the agent are considered dense. While the obstacles smaller than the sensing radius and scattered across the area are considered sparse. Here it is assumed that the agent can modify the sensing radius, as is the case in \cite{sadat1, sadat2} with regards to aerial vehicles. The properties of the altered route recommended by the proposed strategy are outlined below-
\begin{itemize}
    \item This strategy is effective in environments with both dense and sparse obstacles, unlike conventional approaches such as the Lawnmower's path, which struggles in areas with sparse obstacles.
    \item Path adjustments are made within the set of waypoints that have already been visited, making this approach suitable for online use. Moreover, in parts free of obstacles, the strategy follows original SFC pattern, enabling an optimal search in these regions.
    \item Higher iterations of SFCs are preferred for sparse obstacles while lower iterations are preferred for obstacles with larger size. Higher iterations offer a more agile path and lesser occlusion. But, they offer a longer path and may not be desired for dense.
    \item The path generated using the strategy can be adapted for different parts of the region to be explored. Some parts may require more rigorous search than others. In such cases, higher iteration may give the agent more time and focus the search on smaller cells. 
\end{itemize}
The properties mentioned above are illustrated next through different scenarios, as implemented in the code.
\subsection{Results}
The presented strategy was implemented as a Python library. iGraph \cite{igraph} was used for graph operations (version $22.0.4$). Hilbert curve library \cite{hilbert} was used for plotting the Hilbert curve (version $2.0.5$). The implementation and a guide on how to execute the examples can be found at \cite{git_link}.
\begin{itemize}
    \item Dense Obstacles : The strategy was implemented for a situation with normal obstacles. Figure \ref{normal} shows two obstacles blocking the SFC. The modified path is also shown.  
    \par
    Here, while at $0$,  
    \begin{equation*}
    A(V, G) - O = \{1, 2, 3\}
\end{equation*}
the agent searches for obstacle and goes to $1$. Next,
\begin{equation*}
    V = \{0, 1\},   A(V, G) - O = \{2, 3, 13\}
\end{equation*}
agent searches for obstacle and goes to $2$. When no obstacle is found, the agent follows the SFC pattern. This happens till $21$, when obstacle is detected at $22$. Now,
\begin{equation*}
    V = \{0, 1, ..., 21\}
\end{equation*}
\begin{equation*}
   A(V, G) - O = \{23, 29, 30, 31, 32, 53, 54, 57, 58\}
\end{equation*}
So, the agent tries to reach $23$ but encounters obstacle while at $20$. Next, it visits $29$ from $20$ using shortest path. Now,
\begin{equation*}
    V = \{0, 1, ..., 21, 29\}
\end{equation*}
\begin{equation*}
   A(V, G) - O = \{24, 28, 30, 31, 32, 53, 54, 57, 58\}
\end{equation*}
The agent decides to go $24$ but detects an obstacle while at $29$. After which the agent visits $28$,
\begin{equation*}
    V = \{0, 1, ..., 21, 29, 28\}
\end{equation*}
\begin{equation*}
   A(V, G) - O = \{27, 30, 31, 32, 53, 54, 57, 58\}
\end{equation*}
Hence, the agent goes to $27$ and $26$ is added to $A(V, G) - O$, which is reached finally. Now,
\begin{equation*}
    V = \{0, 1, ..., 21, 29, 28, 27, 26\}
\end{equation*}
\begin{equation*}
   A(V, G) - O = \{25, 30, 31, 32, 53, 54, 57, 58\}
\end{equation*}
Next, The agent decides to go $25$ and detects obstacle while at $26$. Finally, the agent goes to $30$, $31$ and so on following the SFC until next obstacle is detected. The search is finally terminated at waypoint $58$.
\begin{figure*}[t]
\centering
\includegraphics[width=2.6in]{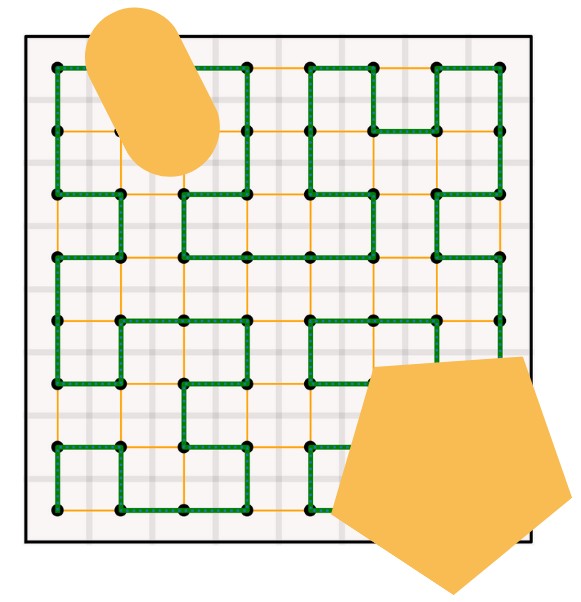}  \raisebox{0.25in}{\includegraphics[width=2.28in]{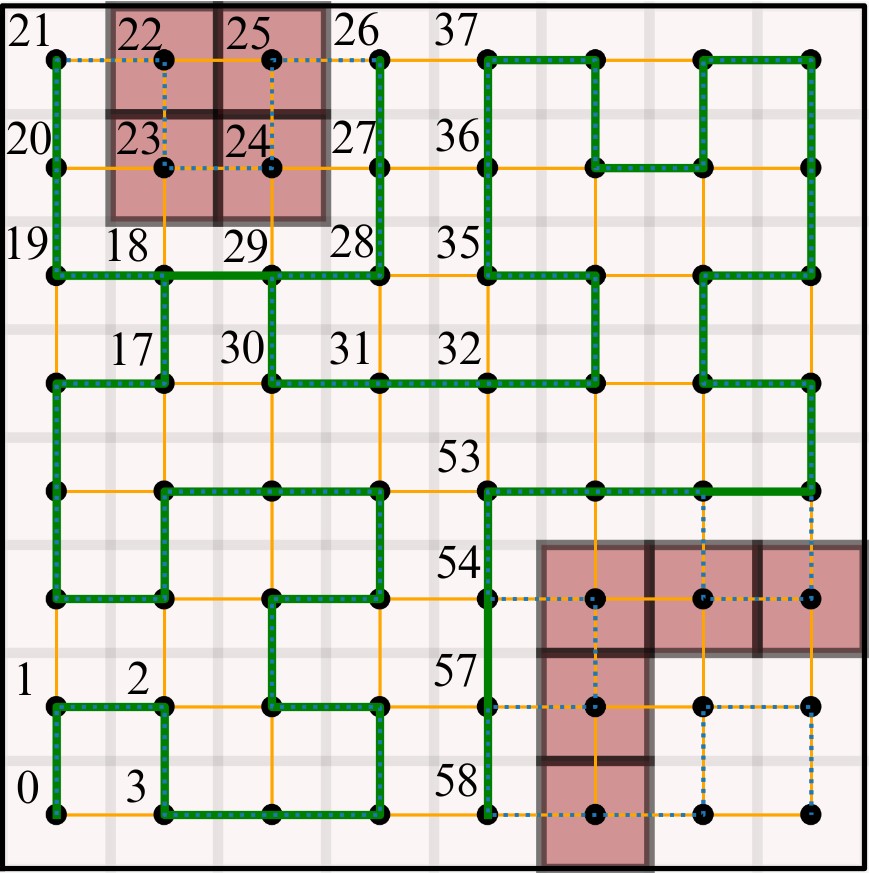}} 
\caption{Hilbert curve with normal sized obstacles; Modified path, Brown represent waypoints with detected obstacle}
\label{normal}
\end{figure*}
    \item Sparse Obstacles : To test the strategy for Sparse obstacles, Iteration $5$ Hilbert curve was blocked by obstacles at unknown random locations. The strategy was able to generate alternate path while following SFC pattern in areas without obstacles as seen in figure \ref{spa}. In scenarios with higher percentage of blocked obstacles resulted in smaller region connected to the starting waypoint. 
\begin{figure*}[t]
\centering
\includegraphics[width=5in]{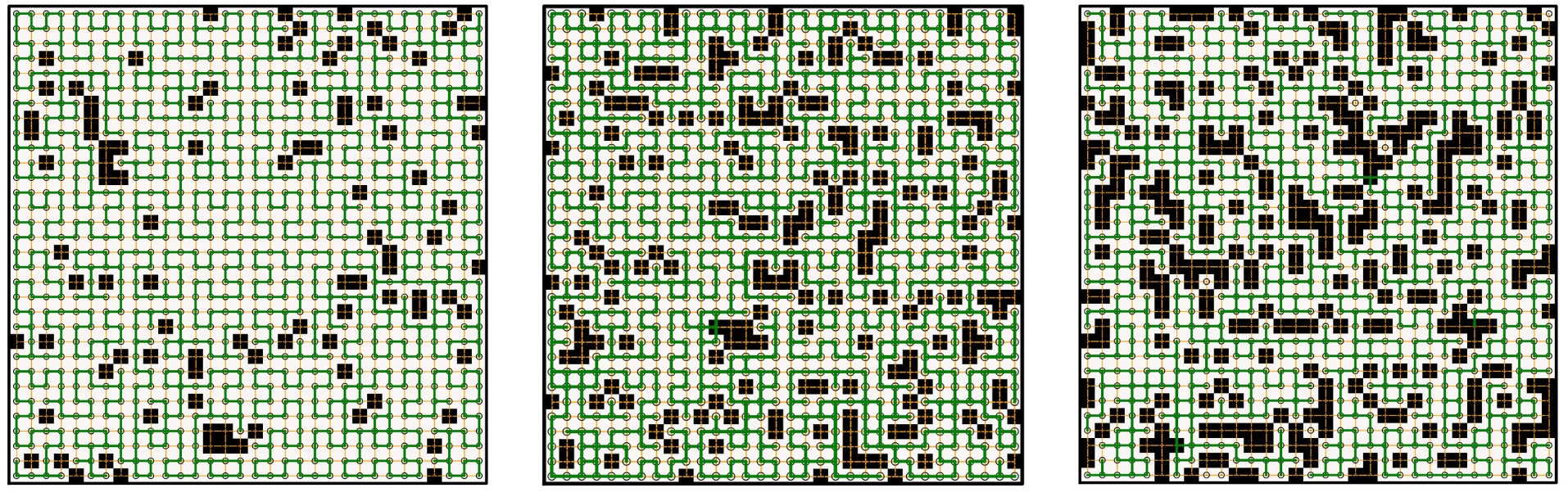}
\caption{Obstacle evasion scenarios with sparse obstacles (shown in black); 10\%, 20\% and 30\% of total waypoints blocked}
\label{spa}
\end{figure*}
    \item Non-Uniform Coverage : In certain cases the user may want to search some parts more rigorously. This is demonstrated by using the proposed strategy in a square region with quadrants spanned by Hilbert curve with different iterations (see figure \ref{non-uni}). Higher iteration (4) is used in a top-left quadrant while lower iteration (1 and 2) in lower quadrants. The proposed strategy suggests exhaustive detour for each of the quadrant when each of them is traversed sequentially.
    \par
    Sometimes, the agent is unable to get to the last waypoint of SFC. This is seen in the right-hand lower quadrant of fig \ref{non-uni}. The agent ends his journey at point $A$. Hence, it is uncertain whether a direct path to the first waypoint of the subsequent SFC is available, as it could be obstructed. Here, agent can be moved to a waypoint ($B$) on the shared edge closest to $A$. Eventually, moving to the closest waypoint on the next SFC and start using the proposed strategy. The approach will guarantee the coverage of waypoints numbered lower than connected to $C$ are visited.

\begin{figure*}
\centering
\includegraphics[width=5.4in]{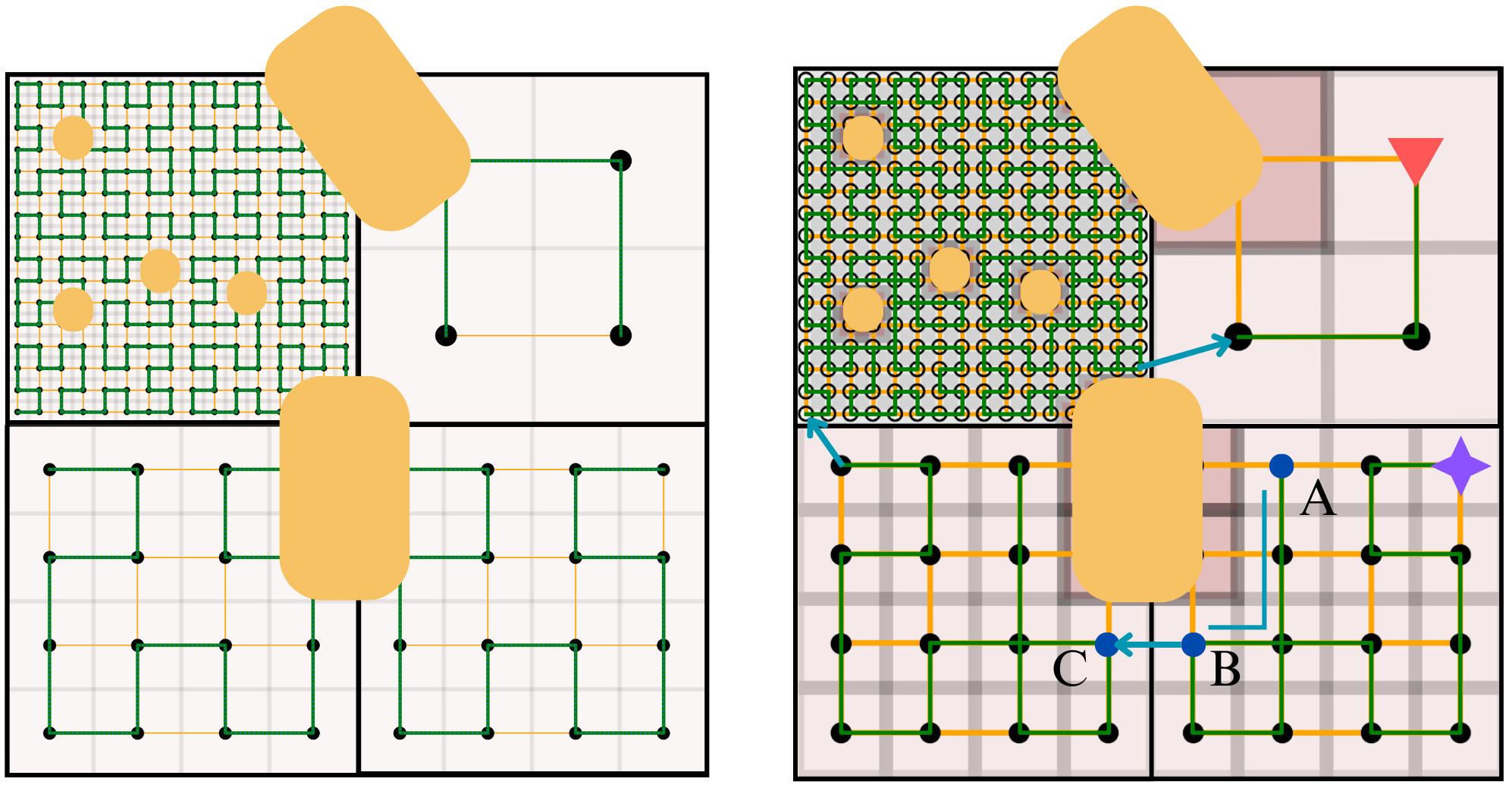}
\caption{Purple star represents starting position, inverted triangle represents terminal position, route between the quadrants is represented by sky-blue arrows}
\label{non-uni}
\end{figure*}
\end{itemize}

\section{Conclusion \& Future Work}
This paper presented a strategy for online obstacle evasion on a Space-Filling curve. The strategy suggests visiting minimum numbered adjacent waypoint that has not been visited. The search was proven to be exhaustive and the strategy was further elaborated through examples with obstacles of different nature. The strategy does not necessarily guarantee a optimal path, but can act as a baseline for further optimization. In addition to addressing the gaps in the strategy, there are other intriguing directions for future work. An enhanced form of the strategy wherein the agent has the knowledge of state (blocked/not-blocked) of all the immediate up-down-left-right nodes could lead to shorter routes as compared with in situ detection as used here. Extension to 3D space and environment with dynamic obstacles could be also be exciting future directions.

\end{document}